\documentclass[letterpaper, 10 pt, conference]{ieeeconf}  

\IEEEoverridecommandlockouts                              

\usepackage{graphics} 
\usepackage[caption=false,font=footnotesize]{subfig} 
\usepackage{amsmath} 
\usepackage{amssymb}  
\usepackage{mathrsfs}
\usepackage{mathtools}

\usepackage{cite}
\usepackage{algorithm}

\makeatletter
\newcommand{\fs@paddedruled}{%
  \fs@ruled
  \def\@fs@pre{\kern5pt\hrule height.8pt depth0pt\kern2pt}%
}
\makeatother
\floatstyle{paddedruled}
\restylefloat{algorithm}

\usepackage{algpseudocode}
\usepackage{color}

\title{\LARGE \bf
Meta-Multi-Agent Reinforcement Learning for Fast Adaptation of Interactive Policies with Applications to Autonomous Driving
}

\author{Huiwen Yan$^{1}$, Kyriakos G. Vamvoudakis$^{2}$, and Mushuang Liu$^{1}$
\thanks{$^{1}$H. Yan and M. Liu are with the Department of Mechanical Engineering, Virginia Tech, Blacksburg, VA 24061, USA
{\tt\small huiweny@vt.edu, mushuang@vt.edu}}
\thanks{$^{2}$K. G. Vamvoudakis is with the Daniel Guggenheim School of Aerospace Engineering, Georgia Tech, Atlanta, GA 30332, USA
{\tt\small kyriakos@gatech.edu}
        }\thanks{This work was supported in part by the DARPA Young Faculty Award with the grant number D24AP00321, by NSF under grant
No.~SLES-2415479, by NASA ULI under grant No.~80NSSC25M7104, and by ARO under
grant No.~W911NF-24-1-0174.}}

\begin{document}

\newtheorem{thm}{Theorem}
\newtheorem{remark}{Remark}
\newtheorem{lemma}{Lemma}
\newtheorem{prop}{Proposition}
\newtheorem{defn}{Definition}
\newtheorem{condi}{Condition}
\newtheorem{assump}{Assumption}

\maketitle
\thispagestyle{empty}
\pagestyle{empty}

\begin{abstract}
This paper develops a meta-multi-agent reinforcement learning (meta-MARL) framework to enable fast adaptation of interactive policies in a multi-agent system (MAS). Meta-reinforcement learning (meta-RL) enables agents to rapidly adapt to new tasks/environments using a bi-level optimization mechanism. However, existing meta-RL generally focuses on single-agent systems. Extending these frameworks and algorithms to multi-agent systems poses additional challenges, as tasks are characterized by not only the environment but also agents’ strategic interactions. To address these challenges, we model multi-agent reinforcement learning (MARL) problems as Markov games (MGs) and develop a meta-MARL framework for rapid interactive policy adaptation across a distribution of MGs. A new concept, called meta-NE, is defined to describe the desired solution concept in a meta-MARL problem. Sufficient conditions for the equivalence between a meta-NE and a stationary point of the gradient-play-based meta-MARL algorithm are established. Our evaluation on autonomous-driving tasks demonstrates that the proposed meta-MARL method achieves faster adaptation than pretrained MARL baselines, validating the effectiveness of our framework.
\end{abstract}

\section{INTRODUCTION}
Reinforcement learning (RL) has been widely applied to complex decision-making problems in diverse fields such as robotics \cite{DroneRacing}, autonomous driving \cite{InterpretableEnd2End, EfficientRLforAD, Think2Drive}, communication \cite{RL_communication}, and smart manufacturing \cite{RL_JobShopSchedule}. RL problems are commonly formulated as Markov decision processes (MDPs), in which an agent interacts with the environment to iteratively optimize its policy \cite{putermanMarkov}. However, when the environment changes, the previously learned policies may not be applicable, and retraining from scratch often requires substantial additional data and computation. It is therefore desirable to develop RL agents that can rapidly adapt to new tasks or environments.

To enable rapid adaptation to new tasks, meta-RL has received increasing attention \cite{MAML, PEARL, Bayes-Adaptive}. Meta-RL learns across a distribution of tasks so that an agent can adapt to a new task more quickly or with less additional data, whereas standard RL optimizes a policy for a single task. Meta-RL is inherently a bilevel optimization problem: the inner level obtains task-specific policies or information, and the outer level maximizes the post-adaptation return. Thus, meta-RL optimizes the learning procedure itself to enable rapid adaptation. Existing meta-RL methods include gradient-based approaches such as MAML \cite{MAML} and context-based approaches such as PEARL \cite{PEARL}. MAML learns a policy-parameter initialization that can adapt to a new task through a small number of gradient updates, and it optimizes this initialization based on the performance of the adapted policies. In contrast, PEARL represents each task using a latent variable inferred from collected data and conditions the policy on this variable. Adaptation therefore updates the posterior distribution of the latent variable rather than the policy parameters. Most existing meta-RL frameworks, however, focus on single-agent settings. In a multi-agent system (MAS), a task is defined not only by the environment but also by strategic interactions among agents, creating additional challenges for rapid policy adaptation.

To characterize strategic interactions in RL problems, multi-agent reinforcement learning (MARL) has recently been investigated \cite{marl-book, gradientPlay}. In MARL, an agent's optimal policy depends not only on the environment but also on the policies of the other agents, which leads naturally to a Markov game (MG) formulation \cite{marl-book}. A standard solution concept in an MG is a Nash equilibrium (NE), at which no agent can improve its return by unilaterally changing its policy. However, finding an NE can be challenging: a deterministic NE may not always exist in a general MG, and learning algorithms, such as gradient play, may not always converge \cite{gradientPlay}. These difficulties motivate the study of Markov potential games (MPGs) \cite{gradientPlay, MAPGInMPG}. MPGs are a special class of MGs in which agents' unilateral incentives are represented by a scalar potential function. Potential games possess favorable equilibrium and learning properties, including the existence of deterministic NEs and the convergence of best-response and better-response dynamics \cite{PotentialGames, MLiu_PotentialGame, MLiu_PCPG}. Under the conditions established in \cite{gradientPlay}, MPGs likewise admit the existence of deterministic NEs and convergence of gradient play to approximate NEs.

To enable rapid adaptation of interactive policies in complex MASs, we develop a meta-MARL framework. As in meta-RL, the inner level obtains task-specific policies, while the outer level optimizes post-adaptation performance. In an MAS, however, a task must capture both the environment and the agents' strategic interactions. Therefore, we formulate each inner-loop optimization problem as an MG problem, and formulate the outer loop such that the NE to every given MG can be quickly derived. Along this direction, a few recent studies have considered possible integrations of meta-RL and MARL \cite{PGA4MARL, MA-Meta-RL}. However, despite recent advances, formal equilibrium concepts for meta-MARL, their relationship to the converged points of a learning algorithm, and their effectiveness in autonomous systems such as autonomous driving, still remain open. The contributions of this paper are fourfold:
\begin{enumerate}
    \item We define the meta-Nash equilibrium (meta-NE) as a solution concept for meta-MARL. At a meta-NE, no agent can improve its expected post-adaptation return by unilaterally changing its own initialization policy.
    \item We establish sufficient conditions under which meta-NEs are equivalent to first-order stationary policies of the induced meta-game.
    \item We develop a MAML-style meta-MARL method that exploits potential functions in MPGs to form a meta-training objective.
    \item We evaluate the proposed method in autonomous highway forced-merging scenarios and perform comprehensive comparisons with pretrained MARL baselines.
\end{enumerate}

The remainder of this paper is organized as follows. Section~\ref{Section2} formulates MGs, MPGs, and MARL. Section~\ref{Section3} defines meta-NE, establishes sufficient conditions for its equivalence to stationary policies, and presents a MAML-style meta-MARL algorithm. Section~\ref{Section4} evaluates the framework in autonomous driving. Section~\ref{Section5} concludes the paper.
\section{Problem Formulation and Preliminaries}\label{Section2}
We formulate Markov games and multi-agent reinforcement learning in Sections~\ref{subsec:MG} and~\ref{subsec:MARL}, respectively.
\subsection{Markov Games}\label{subsec:MG}
We define an MG by the tuple $\mathcal{M}=(\mathcal{N},\mathcal{S},\mathcal{A},P,r,\gamma,\rho)$. The agent set is $\mathcal{N}=\{1,\ldots,N\}$; the global state space is $\mathcal{S}=\mathcal{S}_1\times\cdots\times\mathcal{S}_N$; and the joint action space is $\mathcal{A}=\mathcal{A}_1\times\cdots\times\mathcal{A}_N$, where $\mathcal{S}_i$ and $\mathcal{A}_i$ denote the state and action spaces of agent $i$. The transition kernel $P:\mathcal{S}\times\mathcal{A}\rightarrow\Delta(\mathcal{S})$ gives the probability $P(s'\mid s,a)$ of transitioning from $s=(s_1,\ldots,s_N)$ to $s'$ under the joint action $a=(a_1,\ldots,a_N)$. The tuple $r=(r_1,\ldots,r_N)$ collects the reward functions $r_i:\mathcal{S}\times\mathcal{A}\rightarrow\mathbb{R}$, $\gamma\in[0,1)$ is the discount factor, and $\rho\in\Delta(\mathcal{S})$ is the initial-state distribution.

Actions are selected by the joint policy $\pi:\mathcal{S}\rightarrow\Delta(\mathcal{A})$. We assume that $\pi$ factorizes as $\pi=\pi_1\times\cdots\times\pi_N$, where $\pi_i:\mathcal{S}\rightarrow\Delta(\mathcal{A}_i)$ is agent $i$'s policy. Thus, each agent selects its own action from the observed global state. At time $t$, given $s_t=(s_{1,t},\ldots,s_{N,t})$, the probability of the joint action $a_t=(a_{1,t},\ldots,a_{N,t})$ factorizes as
\begin{equation}\label{eq:generalPolicy}
\pi(a_t|s_t)=\prod_{i=1}^{N}\pi_i(a_{i,t}|s_t).
\end{equation}
We parameterize $\pi$ and $\pi_i$ by $\theta=(\theta_1,\ldots,\theta_N)\in\mathcal{X}$ and $\theta_i\in\mathcal{X}_i$, respectively, and denote the resulting policies by $\pi_{\theta}$ and $\pi_{i,\theta_i}$, where $\mathcal{X}=\mathcal{X}_1\times\cdots\times\mathcal{X}_N$. When the context is clear, we use $\theta$ and $\pi_{\theta}$, and likewise $\theta_i$ and $\pi_{i,\theta_i}$, interchangeably.

For agent $i$, the value function $V_i^{\theta}:\mathcal{S}\rightarrow\mathbb{R}$ is the expected discounted sum of future rewards from initial state $s$:
\begin{equation}
V_i^{\theta}(s)\coloneqq\mathbb{E}\left[\sum_{t=0}^{\infty}\gamma^t r_i(s_t,a_t)\middle|\pi_{\theta}, s_0=s\right].
\end{equation}
The expected return of agent $i$, $J_i:\mathcal{X}\rightarrow\mathbb{R}$, averages its value function over $s_0\sim\rho$:
\begin{equation}\label{eq:totalReward}
J_i(\theta)\coloneqq\mathbb{E}_{s_0\sim\rho}V_i^{\theta}(s_0).
\end{equation}

In an MG, a Nash equilibrium is generally considered a desirable solution concept. Its definition is given below.
\begin{defn}\label{df:nashEquilibrium}(Nash equilibrium \cite{fudenberg1991game})
    A policy $\theta^\star=(\theta_1^\star,\ldots,\theta_N^\star)$ is called a Nash equilibrium if
    \begin{equation}
    J_i(\theta_i^\star,\theta_{-i}^\star)\geq J_i(\tilde{\theta}_i,\theta_{-i}^\star),\quad\forall \tilde{\theta}_i \in \mathcal{X}_i,\quad i \in \mathcal{N}.
    \end{equation}
\end{defn}

\subsection{Multi-Agent Reinforcement Learning}\label{subsec:MARL}
Given the conditions established in \cite{gradientPlay}, we can consider the following gradient play to solve an MG, under which each agent updates its policy according to
\begin{equation}\label{eq:gradientPlay}
    \theta_i^{(k+1)}=\operatorname{proj}_{\mathcal{X}_i}\!\left(\theta_i^{(k)}+\eta\nabla_{\theta_i}J_i(\theta^{(k)})\right),
\end{equation}
where $k$ indexes the learning iteration and $\eta>0$ denotes the learning step size.

Gradient play \eqref{eq:gradientPlay} may not always converge in a general-sum MG. We therefore consider a special subclass, MPGs.

\begin{defn}\label{df:markovPotentialGame}
    (Markov potential game \cite{gradientPlay}). An MG $\mathcal{M}$ is an MPG if there exists a potential function $\phi:\mathcal{S}\times\mathcal{A}\rightarrow\mathbb{R}$ such that, for every $i\in\mathcal{N}$, $(\tilde{\theta}_i,\theta_{-i}),(\theta_i,\theta_{-i})\in\mathcal{X}$, and $s\in\mathcal{S}$, we have
    \begin{equation}\label{eq:markovPotentialGame}
\begin{split}
&\mathbb{E}\Biggl[\sum_{t=0}^{\infty}\gamma^{t}r_i(s_t,a_t)\,\Biggm|\,\pi_{(\tilde{\theta}_i,\theta_{-i})},\,s_0=s\Biggr] \\
&\quad - \mathbb{E}\Biggl[\sum_{t=0}^{\infty}\gamma^{t}r_i(s_t,a_t)\,\Biggm|\,\pi_{(\theta_i,\theta_{-i})},\,s_0=s\Biggr] \\
&= \mathbb{E}\Biggl[\sum_{t=0}^{\infty}\gamma^{t}\phi(s_t,a_t)\,\Biggm|\,\pi_{(\tilde{\theta}_i,\theta_{-i})},\,s_0=s\Biggr] \\
&\quad - \mathbb{E}\Biggl[\sum_{t=0}^{\infty}\gamma^{t}\phi(s_t,a_t)\,\Biggm|\,\pi_{(\theta_i,\theta_{-i})},\,s_0=s\Biggr].
\end{split}
\end{equation}
\end{defn}
The total potential function of an MPG is
\begin{equation}\label{eq:totalPotentialFunction}
\Phi(\theta)\coloneqq\mathbb{E}_{s_0\sim\rho}\left[\sum_{t=0}^{\infty}\gamma^t\phi(s_t,a_t)\Biggm|\pi_{\theta},s_0\right].
\end{equation}

\begin{lemma}
    (Convergence for gradient play \cite{gradientPlay}). Under the conditions established in \cite{gradientPlay}, gradient play \eqref{eq:gradientPlay} reaches an $\varepsilon$-NE of an MPG for any $\varepsilon>0$.
\end{lemma}

Here, an $\varepsilon$-NE is a policy profile at which no agent can improve its expected return by more than $\varepsilon$ through a unilateral deviation.

Given the expected return in \eqref{eq:totalReward} and the total potential in \eqref{eq:totalPotentialFunction}, we can rewrite \eqref{eq:markovPotentialGame} as
\begin{equation}\label{eq:MPGRewrite}
    J_i(\tilde{\theta}_i,\theta_{-i})-J_i(\theta_i,\theta_{-i})=\Phi(\tilde{\theta}_i,\theta_{-i})-\Phi(\theta_i,\theta_{-i}),
\end{equation}
Assuming differentiability, this identity gives
\begin{equation}
    \nabla_{\theta_i}J_i(\theta)=\nabla_{\theta_i}\Phi(\theta).
\end{equation}
Hence, \eqref{eq:gradientPlay} is equivalent to the following:
\begin{equation}
    \theta_i^{(k+1)}=\operatorname{proj}_{\mathcal{X}_i}\!\left(\theta_i^{(k)}+\eta\nabla_{\theta_i}\Phi(\theta^{(k)})\right),
\end{equation}
or, jointly,
\begin{equation}
    \theta^{(k+1)}=\operatorname{proj}_{\mathcal{X}}\!\left(\theta^{(k)}+\eta\nabla_{\theta}\Phi(\theta^{(k)})\right).
\end{equation}
\section{Meta-Multi-Agent Reinforcement Learning}\label{Section3}
We first consider meta-MARL under general MGs in Section~\ref{subsec:Meta-MARL_MG}, define the meta-NE, and study its relationship with first-order stationary policies. We then focus on MPGs in Section~\ref{subsec:Meta-MARL_MPG} and develop an MPG-based MAML-style meta-MARL method.
\subsection{Meta-MARL with General MGs}\label{subsec:Meta-MARL_MG}
The goal of meta-MARL is to enable rapid policy adaptation to any given task in MASs, where a task is represented by an MG characterizing both the environment and the strategic interactions of agents. Specifically, we consider a distribution of MGs, denoted by $p(\mathcal{M})$, from which a task $\mathcal{M}_m$ is sampled. We use $m$ to denote task $\mathcal{M}_m$ when no ambiguity arises.

We first introduce the meta-NE solution concept. Define agent $i$'s expected post-adaptation return over $p(\mathcal{M})$, $\Psi_i\colon\mathcal{X}\rightarrow\mathbb{R}$, by
\begin{equation}
    \Psi_i(\theta)\coloneqq\mathbb{E}_{\mathcal{M}_m\sim p(\mathcal{M})}J_{i,m}(\mathbb{U}_m(\theta)),
\end{equation}
where $\mathbb{U}_m\colon\mathcal{X}\rightarrow\mathcal{X}$ is the joint adaptation rule for task $m$. We consider decentralized adaptation $\mathbb{U}_m=\mathbb{U}_{1,m}\times\cdots\times\mathbb{U}_{N,m}$. Further, we define $\mathbb{U}_{i,m}\colon\mathcal{X}\rightarrow\mathcal{X}_i$ as the adaptation rule, with agent $i$'s adaptation depending on the joint policy $\pi$, which we refer to as a coupled adaptation. This formulation induces the meta-game $\mathcal{G}=(\mathcal{N},\{\mathcal{X}_i\}_{i=1}^{N},\{\Psi_i\}_{i=1}^{N})$ over policy initializations. A meta-NE is a Nash equilibrium of this induced game, at which no agent can improve its expected post-adaptation return by changing its initialization while the other agents' initialization policies remain fixed.
\begin{defn}\label{df:meta-NE} (Meta-Nash equilibrium). A meta-policy initialization $\theta^{\star}=(\theta_1^{\star},\ldots,\theta_N^{\star})$ is a meta-Nash equilibrium with respect to $p(\mathcal{M})$ and $\mathbb{U}_m$ if
\begin{equation}\label{eq:MPG-meta-NE}
    \Psi_i(\theta_i^{\star},\theta_{-i}^{\star})\geq\Psi_i(\tilde{\theta}_i,\theta_{-i}^{\star}),
    \quad \forall \tilde{\theta}_i\in\mathcal{X}_i,\quad i\in\mathcal{N}.
\end{equation}
\end{defn}
\begin{remark}
    If $p(\mathcal{M})$ is concentrated on a single MG and $\mathbb{U}_m(\theta)=\theta$, then the meta-NE condition \eqref{eq:MPG-meta-NE} becomes
    \begin{equation}
    J_i(\theta_i^\star,\theta_{-i}^\star)\geq J_i(\tilde{\theta}_i,\theta_{-i}^\star),
    \end{equation}
    which is the NE condition in Definition~\ref{df:nashEquilibrium}.
\end{remark}

Next, we give sufficient conditions under which meta-NEs are equivalent to first-order stationary policies of the meta-game $\mathcal{G}$.
\begin{defn}
    (First-order stationary policy in meta-MARL). A policy $\theta^{\star}=(\theta_1^\star,\ldots,\theta_N^\star)$ is first-order stationary for the induced meta-game if
    \begin{equation}
    (\tilde{\theta}_i-\theta_i^{\star})^{\top}\nabla_{\theta_i}\Psi_i(\theta^{\star})\leq0,
    \quad \forall \tilde{\theta}_i\in\mathcal{X}_i,\quad i\in\mathcal{N}.
    \end{equation}
\end{defn}
\begin{assump}\label{assump:regularity}
    (Regularity). For each $i\in\mathcal{N}$, $\mathcal{X}_i$ is nonempty, compact, and convex. Moreover, $\Psi_i$ is continuously differentiable on an open set containing $\mathcal{X}$.
\end{assump}
\begin{assump}\label{assump:MGD}
    (Meta-gradient domination). For each $i\in\mathcal{N}$, there exists a finite constant $C_i>0$ such that, for every $\theta\in\mathcal{X}$ and $\tilde{\theta}_i\in\mathcal{X}_i$,
    \begin{small}
    \begin{equation}
        \Psi_i(\tilde{\theta}_i,\theta_{-i})-\Psi_i(\theta)\leq C_i\max_{\bar{\theta}_i\in\mathcal{X}_i}(\bar{\theta}_i-\theta_i)^{\top}\nabla_{\theta_i}\Psi_i(\theta).
    \end{equation}
    \end{small}
\end{assump}
Assumption~\ref{assump:MGD} is analogous to the gradient-domination property established in \cite{gradientPlay}.
\begin{thm}
    Under Assumptions~\ref{assump:regularity} and~\ref{assump:MGD}, a policy is first-order stationary for the induced meta-game if and only if it is a meta-NE.
\end{thm}
\begin{proof}
    Suppose first that $\theta^\star$ is a meta-NE. Fix any $i\in\mathcal{N}$ and $\tilde{\theta}_i\in\mathcal{X}_i$. Since $\mathcal{X}_i$ is convex, $(1-\tau)\theta_i^\star+\tau\tilde{\theta}_i\in\mathcal{X}_i, \forall \tau\in(0,1]$. By the meta-NE condition,
    \begin{equation}
        \Psi_i\bigl((1-\tau)\theta_i^\star+\tau\tilde{\theta}_i,\theta_{-i}^\star\bigr)
        -\Psi_i(\theta^\star)\leq 0.
    \end{equation}
    Dividing by $\tau>0$ and letting $\tau\downarrow0$, differentiability of $\Psi_i$ gives
    \begin{equation}
        (\tilde{\theta}_i-\theta_i^\star)^\top
        \nabla_{\theta_i}\Psi_i(\theta^\star)\leq0,
    \end{equation}
    satisfying first-order stationarity.

    Conversely, suppose that $\theta^\star$ is first-order stationary. Then, for every $i\in\mathcal{N}$,
    \begin{equation}
        (\bar{\theta}_i-\theta_i^\star)^\top
        \nabla_{\theta_i}\Psi_i(\theta^\star)\leq0,
        \qquad \forall \bar{\theta}_i\in\mathcal{X}_i.
    \end{equation}
    Since $\bar{\theta}_i=\theta_i^\star$ is feasible, it follows that $\max_{\bar{\theta}_i\in\mathcal{X}_i}
        (\bar{\theta}_i-\theta_i^\star)^\top
        \nabla_{\theta_i}\Psi_i(\theta^\star)=0$.
    Hence, by Assumption~\ref{assump:MGD}, for any
    $\tilde{\theta}_i\in\mathcal{X}_i$,
    \begin{small}
    \begin{equation}
        \Psi_i(\tilde{\theta}_i,\theta_{-i}^\star)-\Psi_i(\theta^\star)
        \leq
        C_i
        \max_{\bar{\theta}_i\in\mathcal{X}_i}
        (\bar{\theta}_i-\theta_i^\star)^\top
        \nabla_{\theta_i}\Psi_i(\theta^\star)
        =0.
    \end{equation}
    \end{small}
    Therefore, $\Psi_i(\theta^\star)
        \geq
        \Psi_i(\tilde{\theta}_i,\theta_{-i}^\star)$, which is a meta-NE.
\end{proof}
\subsection{Meta-MARL with MPGs}\label{subsec:Meta-MARL_MPG}
We now focus on MPGs. We first establish the potential structure of the induced meta-game under independent adaptation and then formulate an MPG-based, MAML-style meta-MARL framework.

For a distribution of MPGs, define the expected post-adaptation total potential $\Gamma:\mathcal{X}\rightarrow\mathbb{R}$ by
\begin{equation}
    \Gamma(\theta)\coloneqq\mathbb{E}_{\mathcal{M}_m\sim p(\mathcal{M})}\Phi_m(\mathbb{U}_m(\theta)).
\end{equation}
\begin{prop}\label{prop:exactPG}
    For every task $m$, if each agent's adaptation rule is independent of the other agents' policies, i.e.,
\begin{equation}\label{eq:indepAdapt}
\begin{aligned}
    &\mathbb{U}_{i,m}(\theta_i,\theta_{-i})
    =
    \mathbb{U}_{i,m}(\theta_i,\theta_{-i}'),\\
    &\quad
    \forall i\in\mathcal N,\;
    \forall \theta_i\in\mathcal X_i,\;
    \forall \theta_{-i},\theta_{-i}'\in\mathcal X_{-i},
\end{aligned}
\end{equation}
and if the sampled tasks are MPGs with probability one under $p(\mathcal{M})$, then the game $\mathcal{G}$ is an exact potential game with potential function $\Gamma$.
\end{prop}
\begin{proof}
    Under the condition in \eqref{eq:indepAdapt}, the following holds for every sampled MPG (cf. \eqref{eq:MPGRewrite}):
    \begin{small}
    \begin{equation}
    \begin{aligned}
        &J_{i,m}(\mathbb{U}_{m}(\tilde{\theta}_i, \theta_{-i}))-J_{i,m}(\mathbb{U}_{m}({\theta}_i, \theta_{-i}))\\
        &\quad=\Phi_m(\mathbb{U}_{m}(\tilde{\theta}_i, \theta_{-i}))-\Phi_m(\mathbb{U}_{m}({\theta}_i, \theta_{-i})).
    \end{aligned}
    \end{equation}
    \end{small}
    Taking expectations over the task distribution yields
    \begin{equation}
        \Psi_i(\tilde{\theta}_i,\theta_{-i})-\Psi_i(\theta_i,\theta_{-i})=\Gamma(\tilde{\theta}_i,\theta_{-i})-\Gamma(\theta_i,\theta_{-i}),
    \end{equation}
    which completes the proof.
\end{proof}

Together with Assumption~\ref{assump:regularity}, Proposition~\ref{prop:exactPG} implies that
$\Gamma$ attains a global maximum on $\mathcal{X}$, and that every global maximizer
is a meta-NE of $\mathcal{G}$.

In the remainder of this subsection, we study the gradient-based (i.e., MAML-style \cite{MAML}) adaptation rule and parameterize the policy with a neural network (NN). Let $\theta$ denote the parameters of an NN, and therefore, the projection operator can be omitted. We consider one gradient step:
\begin{equation}\label{eq:gradientAdaptation}
\theta_{i,m}'=\mathbb{U}_{i,m}(\theta)=\theta_{i}+\eta_1\nabla_{\theta_{i}}J_{i,m}(\theta),
\end{equation}
where $\eta_1$ is the inner step size.

Because task $m$ is an MPG, \eqref{eq:gradientAdaptation} is equivalent to
\begin{equation}
    \theta_{i,m}'=\theta_{i}+\eta_1\nabla_{\theta_{i}}\Phi_m(\theta),
\end{equation}
and stacking the updates across agents yields
\begin{equation}\label{eq:mpgInner}
    \theta_{m}'=\theta+\eta_1\nabla_{\theta}\Phi_m(\theta),\quad\eta_1>0,
\end{equation}
where $\Phi_m$ is the total potential in task $m$. Equation~\eqref{eq:mpgInner} has the same form as the MAML inner update because an MPG admits a scalar total potential.

The meta-objective therefore becomes
\begin{equation}
\begin{aligned}
    \max_{\theta}\Gamma(\theta)&=\max_{\theta}\mathbb{E}_{\mathcal{M}_m\sim p(\mathcal{M})}
    \Phi_{m}(\theta_m')\\
    &=\max_{\theta}\mathbb{E}_{\mathcal{M}_m\sim p(\mathcal{M})}
    \Phi_{m}\bigl(\theta+\eta_1\nabla_{\theta}\Phi_m(\theta)
    \bigr).
\end{aligned}
\end{equation}
The meta-optimization is performed over $\theta$, whereas the objective is evaluated with the adapted parameters $\theta_m'$. Thus, the MAML-style procedure seeks an initialization from which one or a small number of gradient steps produce high total potential on a new task.

If each $\Phi_m$ is twice differentiable and differentiation can be interchanged with expectation, the exact meta-gradient is
\begin{equation}
\begin{aligned}
\nabla_{\theta}\Gamma(\theta)
&=\mathbb{E}_{\mathcal{M}_m\sim p(\mathcal{M})}\Bigl[
\bigl(\mathbf{I}+\eta_1\nabla_{\theta}^2\Phi_m(\theta)\bigr)^{\top}\\
&\qquad\qquad{}\times\nabla_{\theta_m'}\Phi_m(\theta_m')\Bigr],
\end{aligned}
\end{equation}
where $\mathbf{I}$ is an identity matrix.
Exact-gradient ascent would update $\theta$ according to
\begin{equation}\label{eq:metaUpdate}
    \theta^{(k+1)}=\theta^{(k)}+\eta_2\nabla_{\theta}\Gamma(\theta^{(k)}),\quad\eta_2>0,
\end{equation}
where $\eta_2$ is the outer step size. Following first-order MAML \cite{MAML}, we omit the Hessian term and use the approximation
\begin{equation}\label{eq:firstOrder}
\widetilde{\nabla}_{\theta}\Gamma(\theta)\coloneqq\mathbb{E}_{\mathcal{M}_m\sim p(\mathcal{M})}\left[\nabla_{\theta_m'}\Phi_m(\theta_m')\right].
\end{equation}
Algorithm~\ref{alg:1} summarizes the above implementation.
\begin{algorithm}
\caption{MAML-style MPG-based Meta-Multi-Agent Reinforcement Learning}
\label{alg:1}
 \textbf{Require:}  $p(\mathcal{M})$: distribution over tasks, $\eta_{1},\eta_{2}$: inner and outer step sizes, and $E$: number of meta-training epochs\\
 \textbf{Procedures:} 
\begin{algorithmic}[1]
\State Initialize policy parameter $\theta^{(0)}$.
\For{$k=0,1,\cdots,E-1$}
    \State \parbox[t]{0.85\linewidth}{Sample a batch of tasks $\mathcal{M}_m\sim p(\mathcal{M})$.}
    \For{\textbf{each sampled task} $\mathcal{M}_m$}
    \State \parbox[t]{0.85\linewidth}{Sample $s_0\sim\rho$ and collect the induced trajectories using $\pi_{\theta^{(k)}}$ in $\mathcal{M}_m$.}
    \State \parbox[t]{0.85\linewidth}{Compute $\nabla_{\theta}\Phi_m(\theta^{(k)})$.}
    \State \parbox[t]{0.85\linewidth}{Compute $\theta_m'=\theta^{(k)}+\eta_1\nabla_{\theta}\Phi_{m}(\theta^{(k)})$.}
    \State \parbox[t]{0.85\linewidth}{Sample $s_0'\sim\rho$ and collect the induced trajectories using $\pi_{\theta_m'}$ in $\mathcal{M}_m$.}
    \State \parbox[t]{0.85\linewidth}{Compute $\nabla_{\theta_m'}\Phi_m(\theta_m')$.}
    \EndFor
    \State \parbox[t]{0.95\linewidth}{Compute $\widetilde{\nabla}_{\theta}\Gamma(\theta^{(k)})$ according to \eqref{eq:firstOrder} and update:\\
    $\theta^{(k+1)}=\theta^{(k)}+\eta_2\widetilde{\nabla}_{\theta}\Gamma(\theta^{(k)})$.}
\EndFor
\end{algorithmic}
\end{algorithm}

\section{Applications to Autonomous Driving}\label{Section4}
In this section, we evaluate the MPG-based meta-MARL method in autonomous highway forced-merging scenarios.
\subsection{Simulation Setup}
\vspace{0.25em}
\textbf{(i) Environment and agent configuration}\par
\vspace{0.25em}
Consider the multi-agent highway forced-merging scenario shown in Fig.~\ref{fig:HMMPG}. The on-ramp vehicle is the ego vehicle, and the other vehicles travel in the target lane. We define a vehicle as leading (following) if its longitudinal position is ahead of (behind) that of the ego vehicle along the direction of travel. Then, we select the four nearest leading and four nearest following vehicles relative to the ego vehicle as the game players. Together with the ego vehicle, they form a nine-agent game (i.e., $N=9$).
\begin{figure}
\captionsetup[subfloat]{font=footnotesize}
    \centering
    \includegraphics[width=0.8\linewidth]{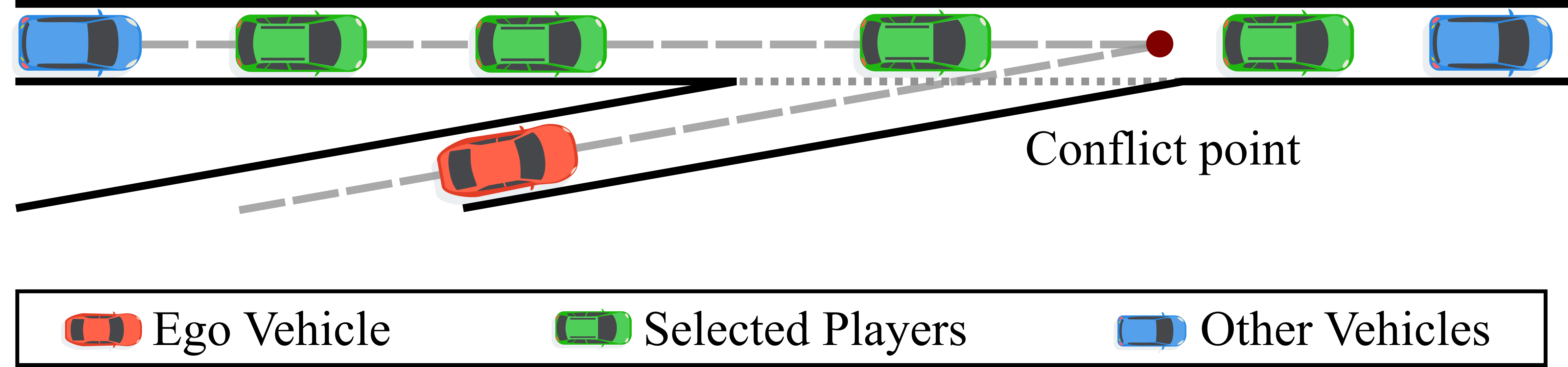}
    \caption{Single-lane highway forced merge scenario.}
    \label{fig:HMMPG}
\end{figure}

We use the kinematic bicycle model \cite{Rajamani2012} to describe the dynamics of each vehicle. At time $t$, vehicle $i$ has control input $[u_{i,t},\,\delta_{i,f,t}]^{\top}$, where $u_{i,t}$ is the longitudinal acceleration and $\delta_{i,f,t}$ is the front-wheel steering angle. The state $s_{i,t}=[x_{i,t},\,y_{i,t},\,v_{i,t},\,\phi_{i,t}]^{\top}$ contains the longitudinal and lateral center-of-mass (CoM) positions in an $x$--$y$ inertial frame, the CoM speed, and the heading angle. The kinematic bicycle model is \cite{Rajamani2012}:
\begin{equation}\label{eq:vehicleDynamics}
\begin{split}
    &x_{i,t+1} = x_{i,t} + v_{i,t} \cos\bigl(\phi_{i,t} + \beta_{i,t}\bigr)\,\Delta t,\\
    &y_{i,t+1} = y_{i,t} + v_{i,t} \sin\bigl(\phi_{i,t} + \beta_{i,t}\bigr)\,\Delta t,\\
    &v_{i,t+1} = v_{i,t} + u_{i,t}\,\Delta t,\\
    &\phi_{i,t+1} = \phi_{i,t} + \frac{v_{i,t}}{l_r} \sin\bigl(\beta_{i,t}\bigr)\,\Delta t,\\
    &\beta_{i,t} = \arctan\Biggl(\frac{l_r}{l_r + l_f}\tan\bigl(\delta_{i,f,t}\bigr)\Biggr),
\end{split}
\end{equation}
where $\beta_{i,t}$ is the sideslip angle between the CoM velocity vector and the vehicle's longitudinal axis, and $l_f$ and $l_r$ are the distances from the CoM to the front and rear axles, respectively. Here, we set $\Delta t=0.5~\mathrm{s}$.

In the simulation, we let each vehicle follow a reference path defined according to the road structure (i.e., vehicles have to be ``on-road"), and thus steering is generated by the reference-path tracking. The action here is thus reduced to a scalar longitudinal acceleration, i.e., $a_{i,t}=u_{i,t}\in[-g,g]$, where $g=9.81~\mathrm{m}/\mathrm{s}^2$. Vehicle speed is constrained to $0~\mathrm{m/s}\leq v_{i,t}\leq30~\mathrm{m/s}$.

\begin{figure}
\vspace*{5pt}
\captionsetup[subfloat]{font=footnotesize}
    \centering
    \includegraphics[width=0.9\linewidth]{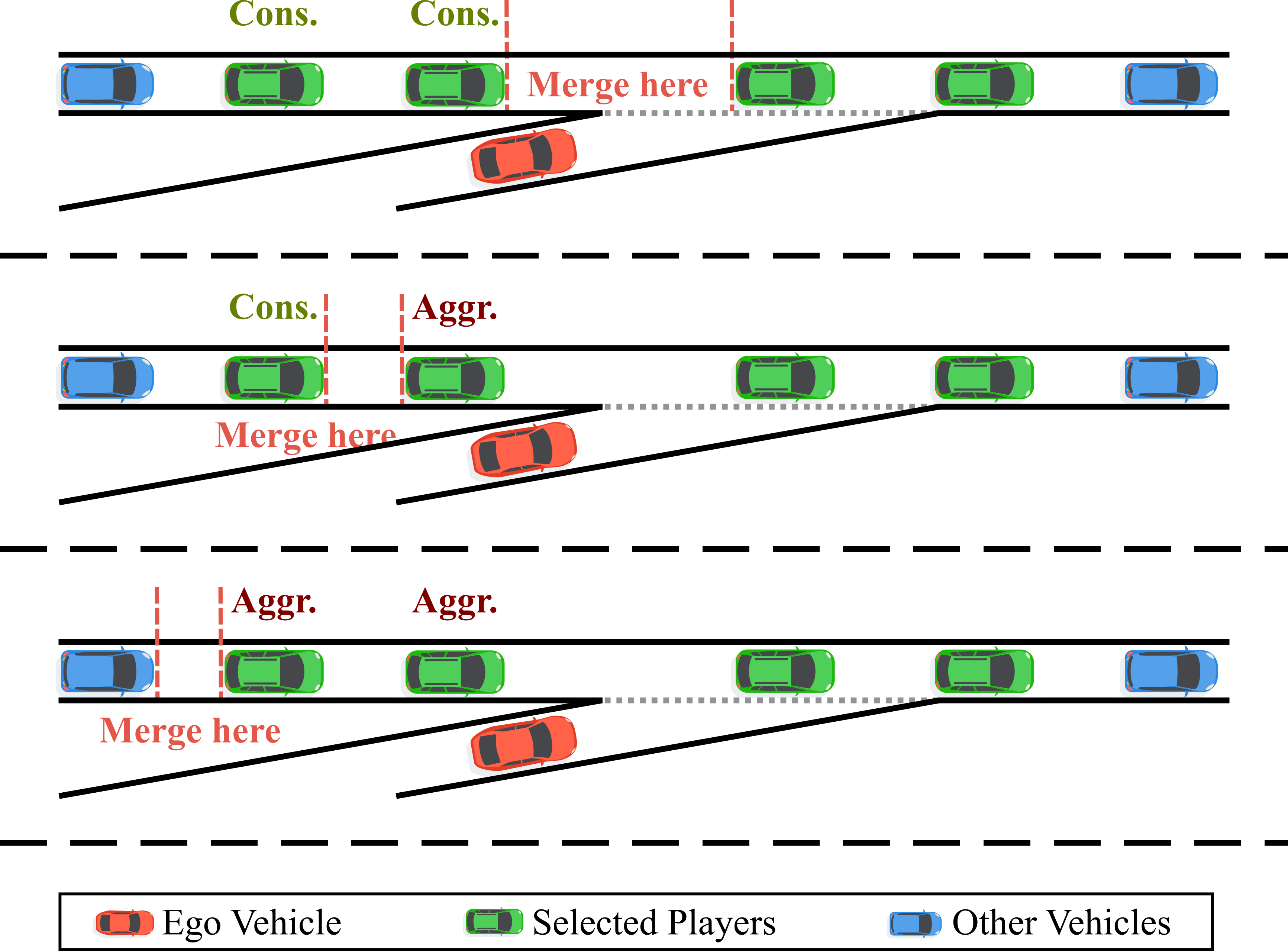}
    \caption{Policy variation under different tasks.}
    \label{fig:HMMetaMARL}
\end{figure}
\begin{figure*}[t]
    \centering

    \subfloat[Vehicles 5--6 aggressive; 7--8 conservative.]{
        \includegraphics[width=0.3\textwidth]{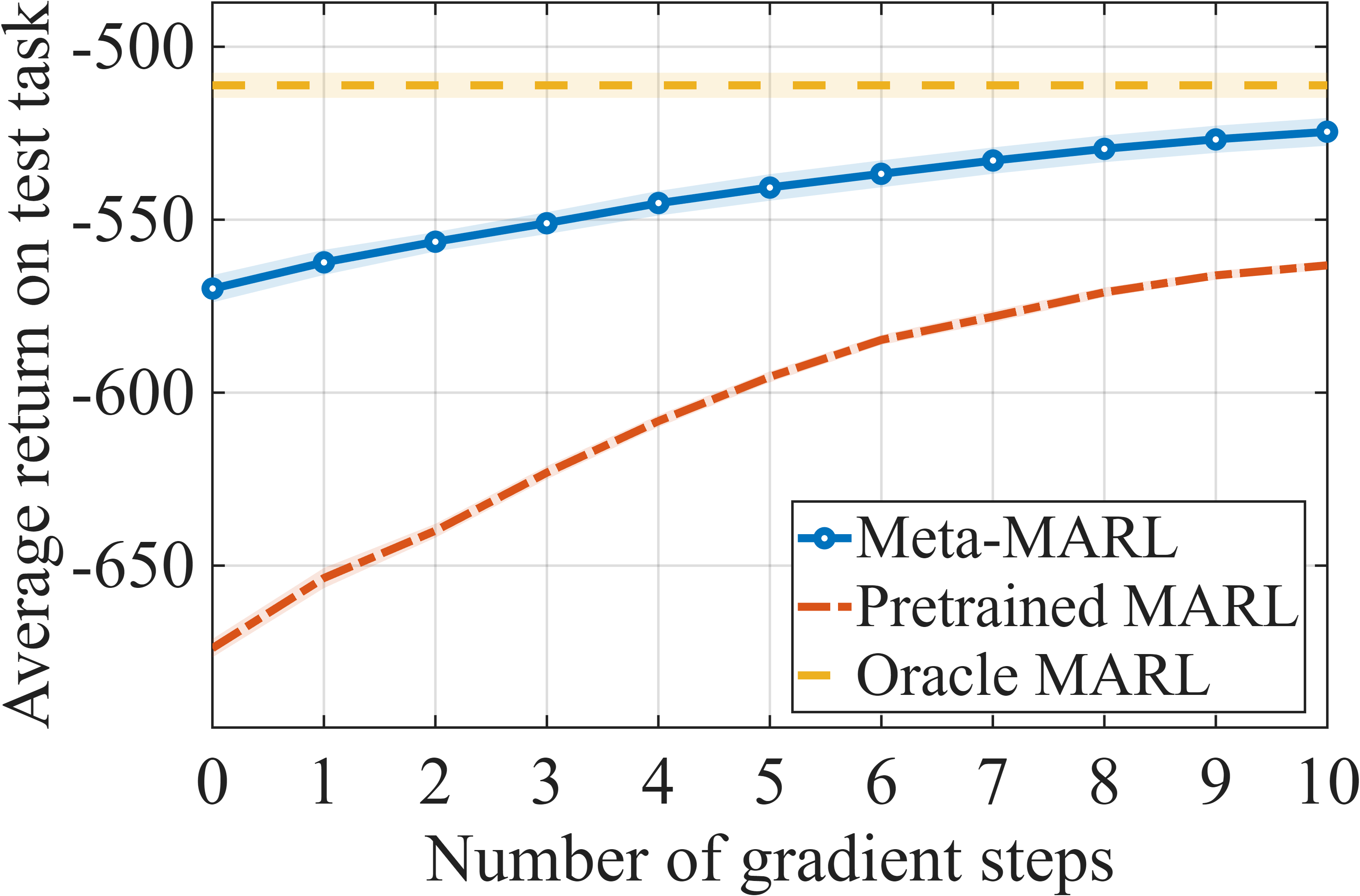}
        \label{fig:Aggr2Cons2}
    }
    \hfill
    \subfloat[Vehicles 5--7 aggressive; 8 conservative.]{
        \includegraphics[width=0.3\textwidth]{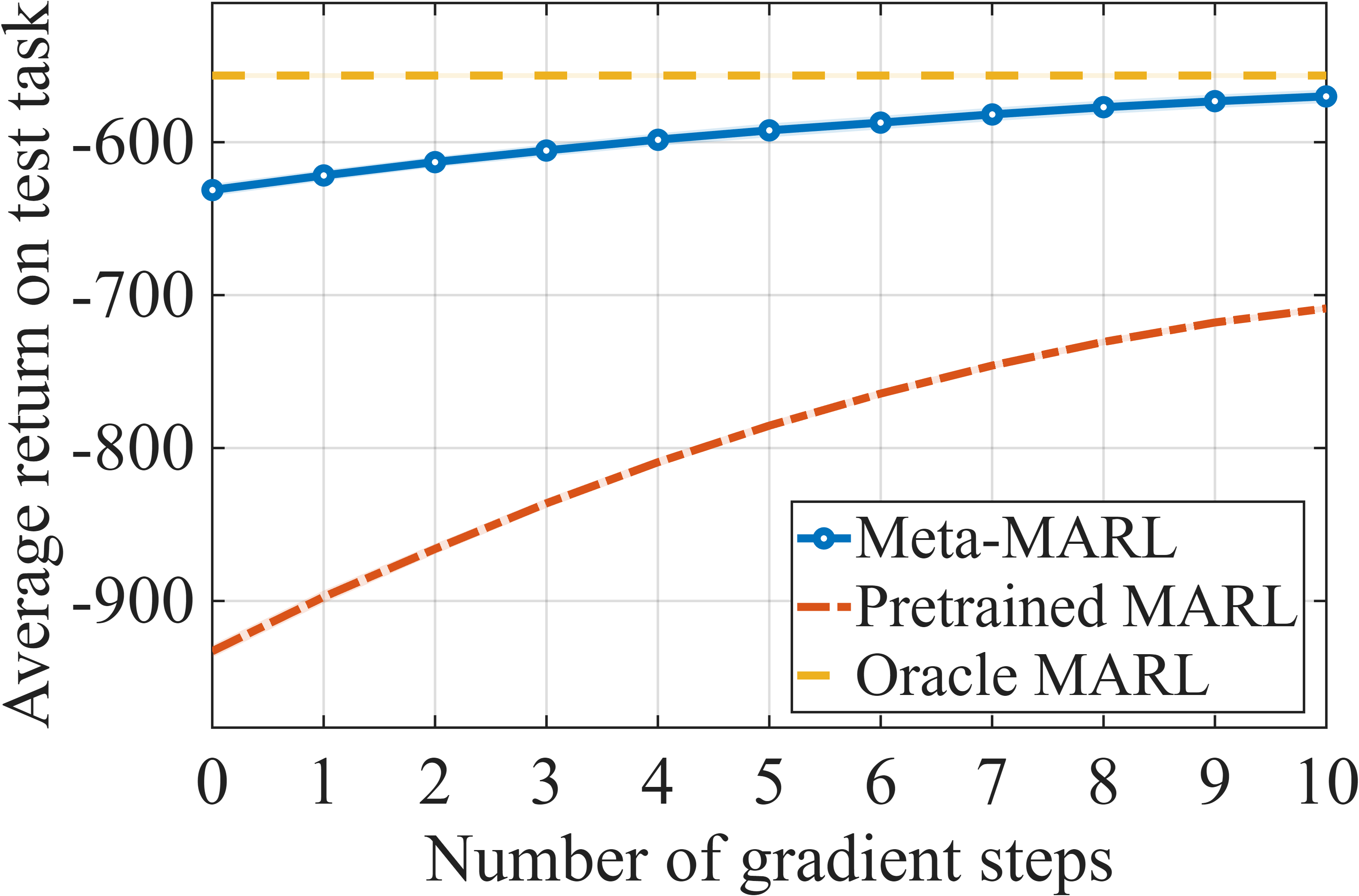}
        \label{fig:Aggr3Cons1}
    }
    \hfill
    \subfloat[Vehicles 5--8 aggressive.]{
        \includegraphics[width=0.3\textwidth]{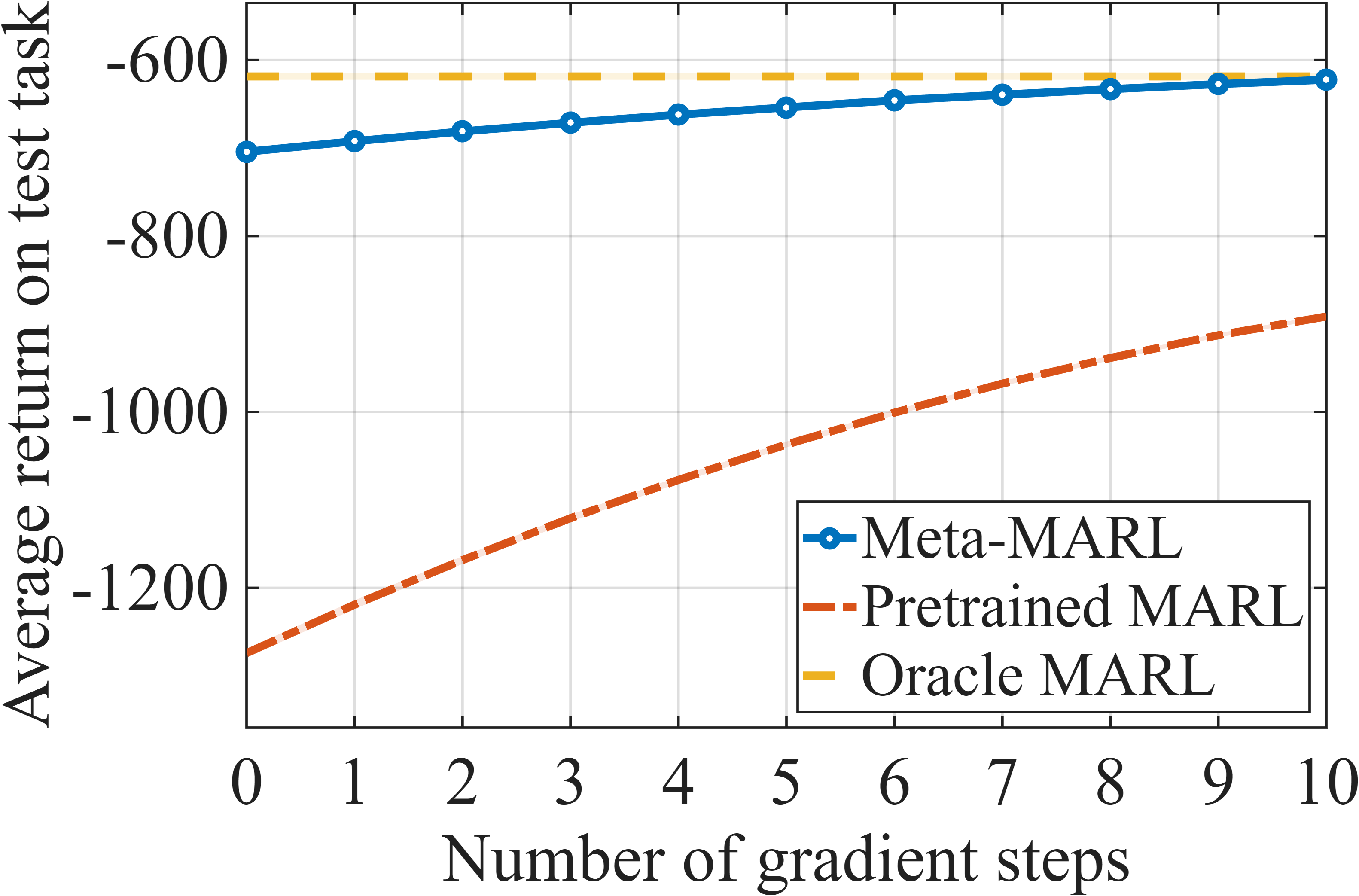}
        \label{fig:Aggr4Cons0}
    }

    \caption{Policy adaptation under different tasks.}
    \label{fig:adaptationCurves}
\end{figure*}

\vspace{0.25em}
\textbf{(ii) Reward function design and task definition}\par
\vspace{0.25em}
We follow the sufficient conditions for MPG construction in \cite{MPGTVT} and design the reward functions accordingly. The ego vehicle is considered in the target lane once it passes the merging conflict point (see Fig. \ref{fig:HMMPG}), whose longitudinal position, $x_c=180~\mathrm{m}$, is determined by the road geometry. 

Let $\mathcal{N}_{i,\mathrm{s}}\subset\mathcal{N}$ and
$\mathcal{N}_{i,\mathrm{c}}\subset\mathcal{N}$ denote the sets of vehicles
traveling in the same lane as vehicle $i$ and in different lanes from vehicle $i$,
respectively. In task $m$, vehicle $i$ receives
\begin{equation}\label{eq:evalTaskReward}
\small
\begin{aligned}
    r_{i,m}(s_t,a_t)&=\alpha_{i,1}^{(m)}r_i^{\mathrm{self},1}+\alpha_{i,2}r_i^{\mathrm{self},2}\\
    &\quad+\sum_{j\in\mathcal{N}_{i,\mathrm{s}}}w_{ij,1}r_{ij}^{\mathrm{joint},1}
    +\sum_{j\in\mathcal{N}_{i,\mathrm{c}}}w_{ij,2}r_{ij}^{\mathrm{joint},2},
\end{aligned}
\end{equation}
where we omit the arguments of each reward component on the right-hand side and consider nonnegative weights. Each term on the right-hand side of \eqref{eq:evalTaskReward} is specified as follows.

The first term rewards desired-speed tracking:
\begin{equation}\label{simeq:self1}
    r_i^{\mathrm{self},1}(s_{i,t},a_{i,t})=-(v_{i,t}-v_{i,d})^2,
\end{equation}
where $v_{i,d}=15~\mathrm{m/s}$ is the desired speed for each vehicle.

The second term rewards small accelerations:
\begin{equation}\label{simeq:self2}
    r_i^{\mathrm{self},2}(s_{i,t},a_{i,t})=-(u_{i,t})^2.
\end{equation}

The third and fourth terms are for collision avoidance. The third term rewards same-lane collision avoidance:
\begin{equation}\label{eq:piecewise}
\begin{aligned}
r_{ij}^{\mathrm{joint},1}(s_{i,t},s_{j,t})&=-\frac{1}{\tau_{ij,t}+\epsilon_{\tau}},\\
    \tau_{ij,t}&=\frac{|x_{i,t}-x_{j,t}|}{\max\{|v_{i,t}-v_{j,t}|,v_c\}},
\end{aligned}
\end{equation}
where $\tau_{ij,t}$ is a timescale for same-lane longitudinal interactions, and $v_c>0$ is a relative-speed threshold and $\epsilon_{\tau}>0$ is a regularizer to prevent division by zero.

The fourth term rewards larger arrival-time gaps between vehicles in different lanes that have not yet passed the conflict point. The reward is designed as:
\begin{equation}\label{simeq:jointReward2}
\begin{split}
    &r_{ij}^{\mathrm{joint},2}(s_{i,t},s_{j,t})\\
          &=-\frac{\mathbf{1}_{\{d_i(t)>0,\;d_j(t)>0\}}}{\sqrt{T_i(t)T_j(t)}(T_i(t)-T_j(t))^2+\epsilon_3},
\end{split}
\end{equation}
where $T_i(t)=d_i(t)/(v_{i,t}+\epsilon_v)$ and $T_j(t)=d_j(t)/(v_{j,t}+\epsilon_v)$ estimate the times required for vehicles $i$ and $j$ to reach the conflict point, and $d_i(t)$ is vehicle $i$'s remaining distance along its reference path. Here, $\epsilon_v>0$ and $\epsilon_3>0$ are to prevent the denominator from being zero. The factor $\sqrt{T_i(t)T_j(t)}$ weights the arrival-time gap by a ``risk level'': for a fixed $T_i(t)-T_j(t)$, a smaller value of $\sqrt{T_i(t)T_j(t)}$ indicates greater risk.

Let $\mathcal{I}_{\mathrm{s}}$ be the undirected set of same-lane interaction pairs and $\mathcal{I}_{\mathrm{c}}$ the undirected set of different-lane interaction pairs with merging conflicts. Hence, the potential for task $m$ is
\begin{equation}
\small
\begin{aligned}
\phi_m(s_t,a_t)
&=\sum_{i\in\mathcal{N}}
\left[\alpha_{i,1}^{(m)}r_i^{\mathrm{self},1}
+\alpha_{i,2}r_i^{\mathrm{self},2}\right]\\
&\quad+\sum_{\{i,j\}\in\mathcal{I}_{\mathrm{s}}}
w_{ij,1}r_{ij}^{\mathrm{joint},1}
+\sum_{\{i,j\}\in\mathcal{I}_{\mathrm{c}}}
w_{ij,2}r_{ij}^{\mathrm{joint},2},
\end{aligned}
\end{equation}
where we omit the arguments on the right-hand side.

Given the above setup, the component $r$ in $\mathcal{M}=(\mathcal{N},\mathcal{S},\mathcal{A}, P,r,\gamma,\rho)$ is task-dependent. Hence, we define the task distribution by varying the weighting parameters
$\{\alpha_{i,1}^{(m)}\}_{i\in\mathcal{N}\setminus\{1\}}$ of the surrounding
vehicles, while fixing the ego vehicle's weight $\alpha_{1,1}^{(m)}$. In other words, we use the eight-dimensional vector $[\alpha_{2,1}^{(m)},\ldots,\alpha_{9,1}^{(m)}]^{\top}$, to specify a task. Here we select the parameter $\alpha_{i,1}^{(m)}$ because it characterizes vehicle $i$'s driving aggressiveness: a larger $\alpha_{i,1}^{(m)}$ indicates a stronger desire to keep its speed (or equivalently, less emphasis on maintaining a safe distance). 
\begin{remark}
In general, any subset of the parameters within the reward function, including the weighting parameters and the desired speeds, may be varied to construct $p(\mathcal{M})$.
\end{remark}

\vspace{0.25em}
\textbf{(iii) Policy network and parameter sharing}\par
\vspace{0.25em}
Each vehicle is controlled by a deterministic policy represented by a NN. Because the vehicles have identical dynamics, we adopt parameter sharing \cite{parameterSharing, scalingWithMARL} to improve training efficiency: all agents share a common policy network, while a one-hot vehicle index enables distinct actions. We note that such an architecture may restrict the parameter space and lead to potential suboptimality; however, the potential suboptimality is not observed in our simulations. For vehicle $i$, the network receives a global observation containing its normalized position, velocity, and lane information; the relative positions, velocities, and lane information of all other vehicles; and a one-hot encoding of its index. As shown in Fig. \ref{fig:policyNN}, the policy network comprises an input layer, two hidden layers, and an output layer:
\begin{enumerate}
    \item \textbf{Input layer:} a global observation of all vehicles.
    
    \item \textbf{Hidden layers:} two fully connected layers, each followed by a leaky rectified linear unit (Leaky ReLU) activation function.
    
    \item \textbf{Output layer:} a fully connected layer followed by a hyperbolic tangent (tanh) activation function. The output is then scaled by $9.81~\mathrm{m/s^2}$.
\end{enumerate}
\begin{figure}[htbp]
\captionsetup[subfloat]{font=footnotesize}
    \centering
    \includegraphics[width=0.95\linewidth]{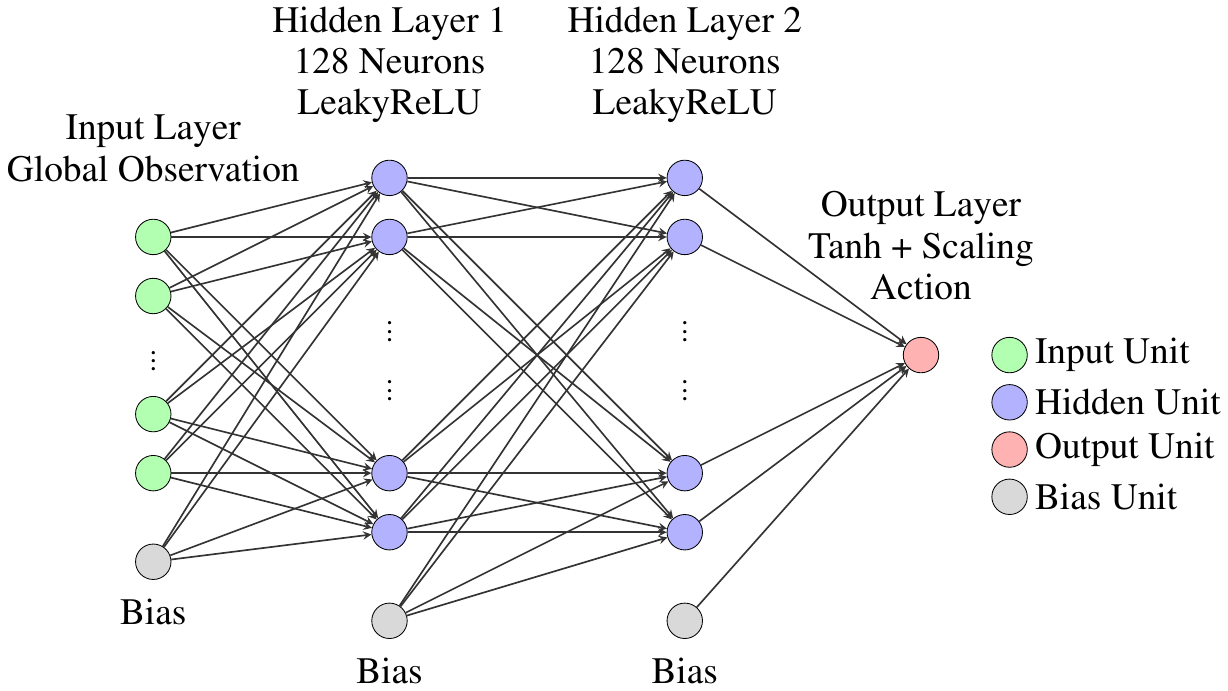}
    \caption{Policy network architecture.}
    \label{fig:policyNN}
\end{figure}
\begin{figure}[htbp]
\captionsetup[subfloat]{font=footnotesize}
    \centering
    \includegraphics[width=0.65\linewidth]{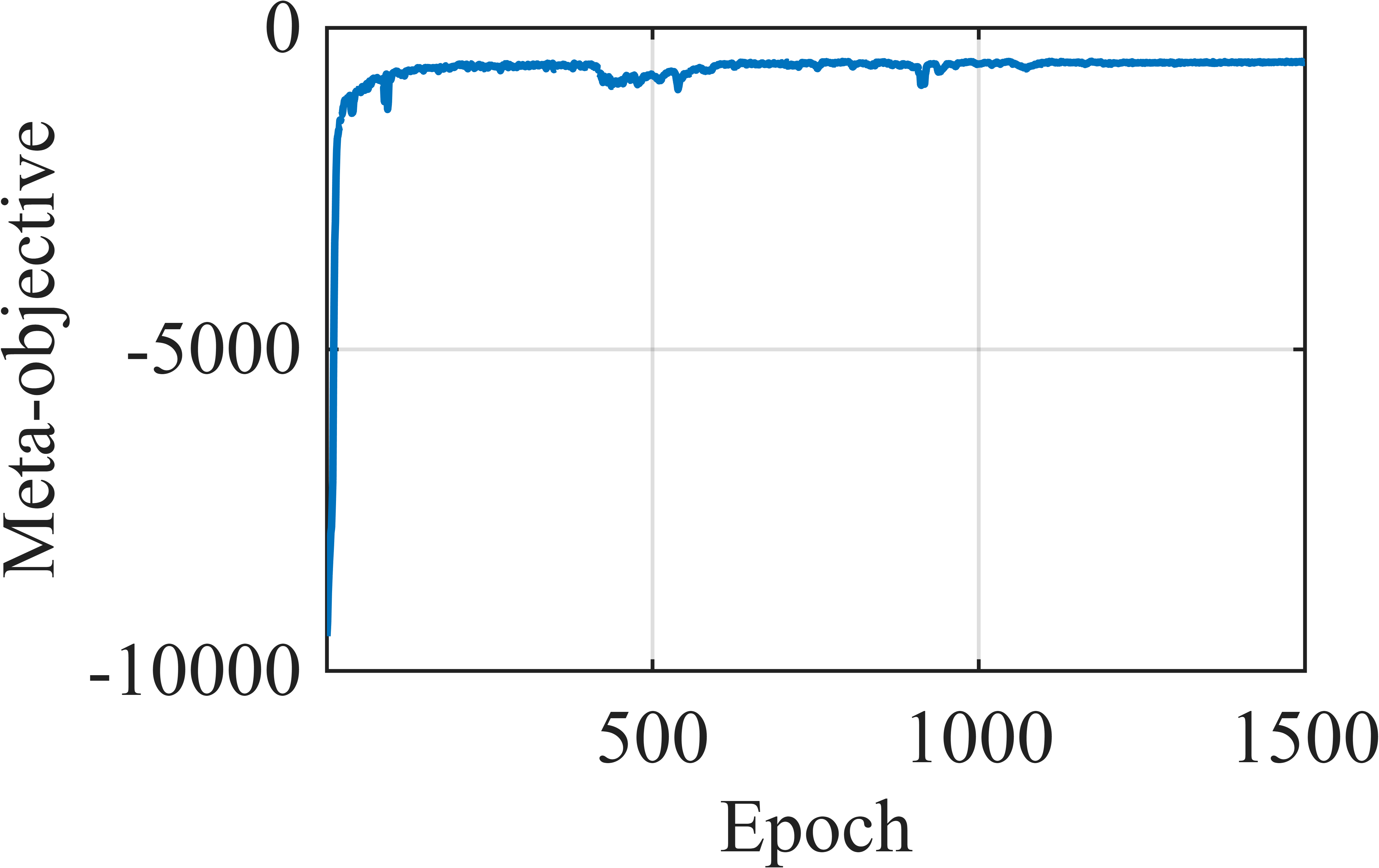}
    \caption{Convergence of policy during meta-training.}
    \label{fig:convergence}
\end{figure}

\vspace{0.25em}
\textbf{(iv) Training setup and convergence results}\par
\vspace{0.25em}
During meta-training, we set $\gamma=0.99$ and estimate $\Gamma$ with a finite horizon of $H=60$ decision steps, which corresponds to $H\Delta t=30~\mathrm{s}$:
\begin{equation}
\small
\begin{aligned}
    \Gamma^H(\theta)&\coloneqq\mathbb{E}_{\mathcal{M}_m\sim p(\mathcal{M})}\Phi_{m}^H(\theta_m')\\
    &=\mathbb{E}_{\mathcal{M}_m\sim p(\mathcal{M})}\mathbb{E}_{s_0\sim\rho}\left[\sum_{t=0}^{H-1}\gamma^t\phi_m(s_t,a_t)\Biggm|\pi_{\theta_m'},s_0\right].
\end{aligned}
\end{equation}

The initial states are generated using stratified sampling to ensure broad coverage of the state space. Vehicle longitudinal positions and velocities are sampled within predefined ranges. For each pair of consecutive vehicles in the same lane, the initial longitudinal CoM distance is at least $7~\mathrm{m}$, and, when the rear vehicle is closing on the vehicle ahead, the time-to-collision (TTC), computed as the longitudinal CoM distance divided by the closing speed, is at least $4~\mathrm{s}$.
\begin{figure*}[!t]
\vspace*{5pt}
\captionsetup[subfloat]{font=footnotesize}
\centering
\subfloat[]{
  \centering
  \includegraphics[width=0.90\textwidth,height=0.20\textheight,keepaspectratio]{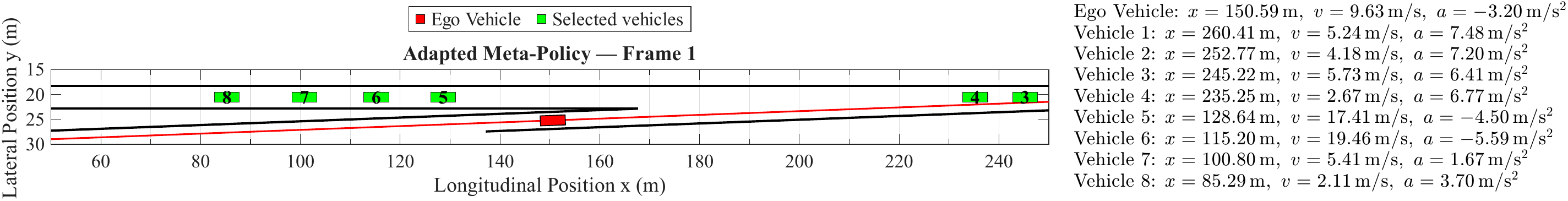}
  \label{fig:Task11}
}
\par
\subfloat[]{
  \centering
  \includegraphics[width=0.90\textwidth,height=0.20\textheight,keepaspectratio]{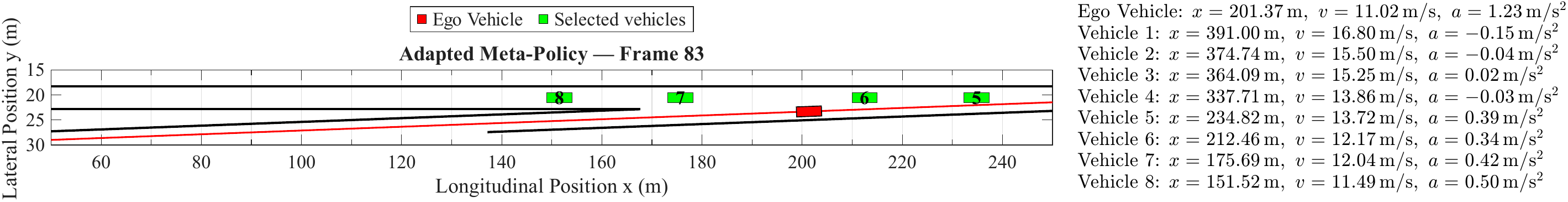}
  \label{fig:Task12}
}
\caption{Illustration of the merging behavior in Task 1. (a) Ego vehicle before merging. (b) Ego vehicle after merging.}
\label{fig:Task1}
\end{figure*}
\begin{figure*}[!t]
\captionsetup[subfloat]{font=footnotesize}
\centering
\subfloat[]{
  \centering
  \includegraphics[width=0.90\textwidth,height=0.20\textheight,keepaspectratio]{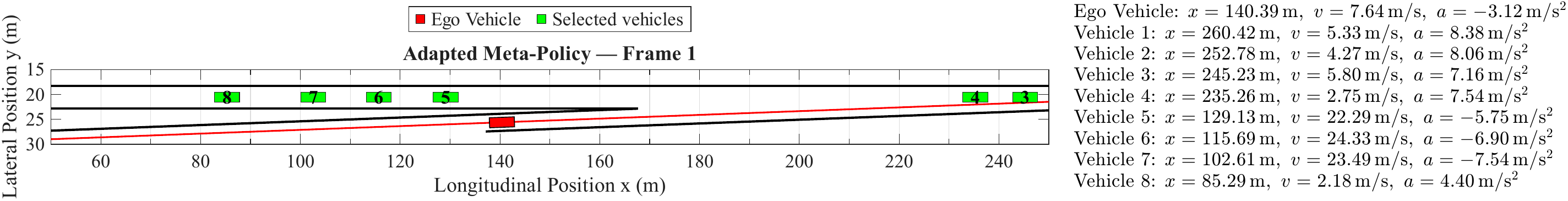}
  \label{fig:Task21}
}
\par
\subfloat[]{
  \centering
  \includegraphics[width=0.90\textwidth,height=0.20\textheight,keepaspectratio]{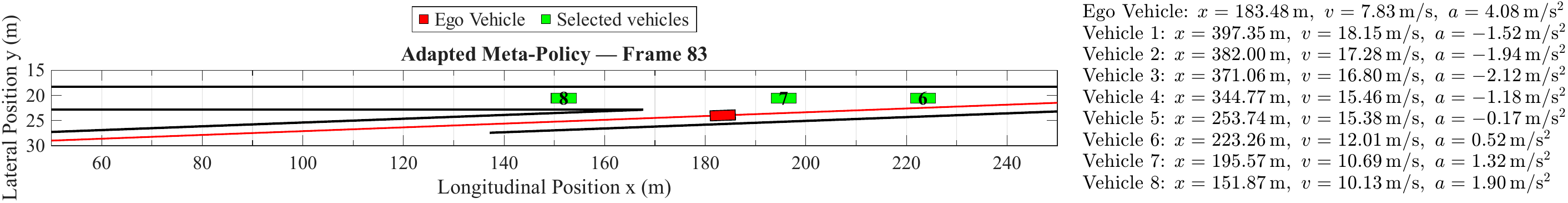}
  \label{fig:Task22}
}
\caption{Illustration of the merging behavior in Task 2. (a) Ego vehicle before merging. (b) Ego vehicle after merging.}
\label{fig:Task2}
\end{figure*}
\begin{figure*}[!t]
\captionsetup[subfloat]{font=footnotesize}
\centering
\subfloat[]{
  \centering
  \includegraphics[width=0.90\textwidth,height=0.20\textheight,keepaspectratio]{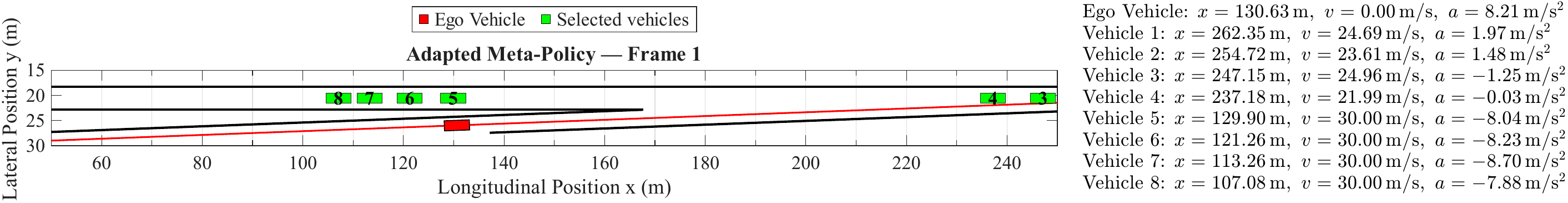}
  \label{fig:Task31}
}
\par
\subfloat[]{
  \centering
  \includegraphics[width=0.90\textwidth,height=0.20\textheight,keepaspectratio]{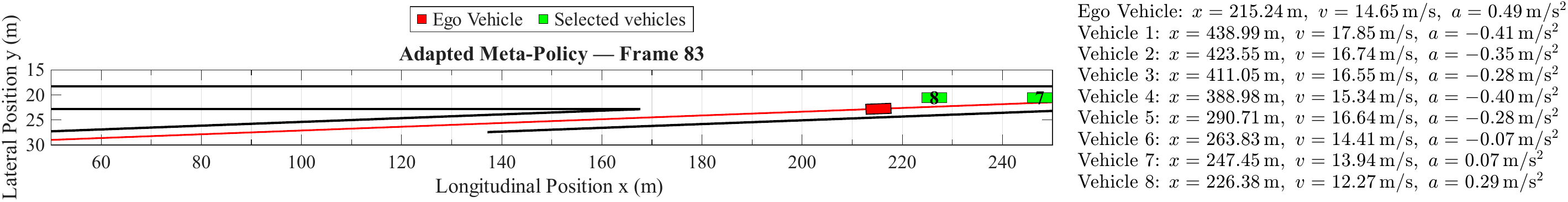}
  \label{fig:Task32}
}
\caption{Illustration of the merging behavior in Task 3. (a) Ego vehicle before merging. (b) Ego vehicle after merging.}
\label{fig:Task3}
\end{figure*}

Meta-training is performed for 1500 epochs. At each meta-training epoch, we sample 16 tasks and perform one inner-loop stochastic gradient ascent (SGA) step with learning rate $\eta_1=10^{-3}$ using trajectories induced by the meta-policy to obtain the adapted policies. We then evaluate the adapted policies to estimate the first-order meta-gradient and update the meta-policy in the outer-loop using the adaptive moment estimation (Adam) optimizer \cite{Adam} with learning rate $\eta_2=10^{-3}$. Fig.~\ref{fig:convergence} plots the estimated meta-objective $\Gamma^H$ during meta-training and shows empirical convergence over the reported training run.
\begin{table*}[htbp]
\centering
\caption{Statistical test results for the evaluated policies}
\label{tab:quantitativeResults}
\begin{tabular}{c|c|c|c} 
\hline
Performance metric & Reference policy & Adapted meta-policy & Adapted pretrained policy  \\ \hline
Collision frequency & 0/5000 & 0/5000 & 6/5000 \\ \hline
Mean minimum inter-vehicle distance (m) & 14.9540 & 14.9783 & 13.1770\\ \hline
Mean ego-vehicle speed ($\mathrm{m/s}$) & 11.7712 & 12.1966 & 11.4907\\ \hline
Mean absolute ego acceleration ($\mathrm{m}/\mathrm{s}^{2}$) & 0.4586 & 0.5912 & 0.6424\\ \hline
Mean absolute ego jerk ($\mathrm{m}/\mathrm{s}^{3}$) & 0.2446 & 0.6424 & 0.1169\\ \hline
\end{tabular}
\end{table*}
\subsection{Task-Specific Evaluation}\label{subsec:specific}
To evaluate the adaptation performance of the learned meta-policy for the ego vehicle when interacting with surrounding vehicles exhibiting different driving styles, we consider three specific tasks (see Fig.~\ref{fig:HMMetaMARL}).  In each task, the vehicle policies are adapted from: (a) the \textbf{\textit{meta-policy}} initialization and (b) a \textbf{\textit{pretrained policy}}, which is an NE in a task different from the target test task. After initialization, both policies adapt according to gradient ascent. We compare the adapted policies with (c) an \textbf{\textit{oracle policy}}, which is the NE in the target test task. The ``average return on test task'' in Fig.~\ref{fig:adaptationCurves} refers to the average estimated total potential $\Phi_{m}^H$ evaluated on the target test task. The ego vehicle behaviors described below are from the adapted meta-policy.

\textit{Task 1:} Vehicles 5 and 6 exhibit aggressive behavior with large $\alpha_{i,1}$, while vehicles 7 and 8 exhibit conservative behavior with small $\alpha_{i,1}$. As the gap between vehicles 6 and 7 opens, the ego decelerates to let vehicles 5 and 6 pass, then merges into the gap. Two key moments are shown in Fig.~\ref{fig:Task1}, and the corresponding adaptation curves in Fig.~\ref{fig:adaptationCurves}\subref{fig:Aggr2Cons2}.

\textit{Task 2:} Vehicles 5, 6, and 7 exhibit aggressive behavior, while vehicle 8 exhibits conservative behavior. The ego vehicle first decelerates to let vehicles 5--7 pass and then merges into the gap between vehicles 7 and 8. Two key moments are shown in Fig.~\ref{fig:Task2}. The corresponding adaptation curves are presented in Fig.~\ref{fig:adaptationCurves}\subref{fig:Aggr3Cons1}.

\textit{Task 3:} Vehicles 5--8 all exhibit aggressive behavior. The ego vehicle decelerates to let them pass, then merges behind vehicle 8. Two key moments are shown in Fig.~\ref{fig:Task3}, with the corresponding adaptation curves in Fig.~\ref{fig:adaptationCurves}\subref{fig:Aggr4Cons0}.

Figs.~\ref{fig:Task1}--\ref{fig:Task3} show that the meta-policy can enable the ego vehicle to quickly adapt to various surrounding vehicles' driving styles and perform reasonably well in terms of ``where" and ``how" to merge. Fig.~\ref{fig:adaptationCurves}, in addition, shows that (i) the meta-policy adapts faster than the pretrained policy, and that (ii) the meta-policy converges to the optimal (i.e., the oracle) within around $10$ gradient steps.
\subsection{Statistical Evaluation}
To provide a more comprehensive evaluation and to assess the robustness of the meta-policy, we conduct a statistical comparison. Specifically, we evaluate 50 randomly selected tasks, and for each task, we consider 100 randomly selected scenarios (i.e., initial states). Therefore, in total, $5000$ random scenarios are tested. The performance of the (a) meta-policy, (b) pretrained policy, and (c) oracle policy is shown in Table~\ref{tab:quantitativeResults}. In Table~\ref{tab:quantitativeResults}, a collision occurs when the ego vehicle contacts a surrounding vehicle during a test scenario. The remaining four metrics are computed for each scenario and averaged over all $5000$ scenarios. We observe the following:
\begin{enumerate}
    \item \textbf {\textit {Safety:}} No collisions were observed for the adapted meta-policy, whereas 6 collisions were observed for the pretrained policy. In addition, the adapted meta-policy leads to a larger minimum inter-vehicle distance. Both metrics suggest that the adapted meta-policy performs better than the pretrained policy in terms of safety. 
    \item \textbf {\textit {Travel efficiency:}} The adapted meta-policy leads to a higher ego-vehicle speed than the pretrained policy, indicating that the better safety performance does not compromise the travel efficiency.
    \item \textbf {\textit {Energy efficiency:}} The adapted meta-policy exhibited lower absolute acceleration than the pretrained policy, indicating less control effort.
    \item \textbf {\textit {Ride comfort:}} The adapted meta-policy exhibited a higher absolute jerk than the pretrained policy, indicating that the better safety, travel efficiency, and energy efficiency may come at the price of ride comfort.
\end{enumerate}

\section{CONCLUSIONS}\label{Section5}
Meta-MARL enables rapid adaptation of interactive policy across tasks in MASs, including changes in both the environment and agents' interaction patterns. We defined a meta-NE to represent the desired policy profile in which no agent can improve its expected post-adaptation return through a unilateral change. We established sufficient conditions under which meta-NEs are equivalent to first-order stationary policies of the induced meta-game. For MPG tasks with independent adaptation, we showed that the induced meta-game has an exact potential; for coupled MAML-style adaptation, we formulated a MPG-based meta-MARL framework using a first-order MAML approximation. Numerical studies in autonomous driving applications show that the developed meta-policy exhibits fast adaptations to various driving styles of the surrounding vehicles. Statistical studies suggest that the meta-policy performs better than the pretrained policy in terms of safety, travel efficiency, and energy efficiency. 



\bibliographystyle{IEEEtran}
\bibliography{Reference}

@inproceedings{Adam,
  author    = {Kingma, Diederik P. and Ba, Jimmy},
  title     = {Adam: A Method for Stochastic Optimization},
  booktitle = {International Conference on Learning Representations (ICLR)},
  year      = {2015},
  url       = {https://arxiv.org/abs/1412.6980}
}

@inproceedings{MAPGInMPG,
title={Global Convergence of Multi-Agent Policy Gradient in {Markov} Potential Games},
author={Stefanos Leonardos and Will Overman and Ioannis Panageas and Georgios Piliouras},
booktitle={International Conference on Learning Representations},
year={2022},
url={https://openreview.net/forum?id=gfwON7rAm4}
}

@inproceedings{Bayes-Adaptive,
  title = "VariBAD: A Very Good Method for Bayes-Adaptive Deep RL via Meta-Learning",
  author = "Luisa Zintgraf and Kyriacos Shiarlis and Maximilian Igl and Sebastian Schulze and Yarin Gal and Katja Hofmann and Shimon Whiteson",
  year = "2020",
  booktitle = "ICLR 2020: Proceedings of the Eighth International Conference on Learning Representations",
  month = "May"
}

@inproceedings{RL_JobShopSchedule,
 author = {Zhang, Cong and Song, Wen and Cao, Zhiguang and Zhang, Jie and Tan, Puay Siew and Chi, Xu},
 booktitle = {Advances in Neural Information Processing Systems},
 pages = {1621--1632},
 publisher = {Curran Associates, Inc.},
 title = {Learning to Dispatch for Job Shop Scheduling via Deep Reinforcement Learning},
 volume = {33},
 year = {2020}
}

@InProceedings{scalingWithMARL,
  title = 	 {Scaling Multi-Agent Reinforcement Learning with Selective Parameter Sharing},
  author =       {Christianos, Filippos and Papoudakis, Georgios and Rahman, Muhammad A and Albrecht, Stefano V},
  booktitle = 	 {Proceedings of the 38th International Conference on Machine Learning},
  pages = 	 {1989--1998},
  year = 	 {2021},
  volume = 	 {139},
  month = 	 {18--24 Jul},
}

@misc{parameterSharing,
      title={Revisiting Parameter Sharing in Multi-Agent Deep Reinforcement Learning}, 
      author={J. K. Terry and Nathaniel Grammel and Sanghyun Son and Benjamin Black and Aakriti Agrawal},
      year={2023},
      note={arXiv:2005.13625}
}

@ARTICLE{MLiu_PCPG,
  author={Liu, Mushuang and Tseng, H. Eric and Filev, Dimitar and Girard, Anouck and Kolmanovsky, Ilya},
  journal={IEEE Transactions on Control Systems Technology}, 
  title={Safe and Human-Like Autonomous Driving: A Predictor–Corrector Potential Game Approach}, 
  year={2024},
  volume={32},
  number={3},
  pages={834-848},
  doi={10.1109/TCST.2023.3332438}}

@InProceedings{PGA4MARL,
  title = 	 {A Policy Gradient Algorithm for Learning to Learn in Multiagent Reinforcement Learning},
  author =       {Kim, Dong Ki and Liu, Miao and Riemer, Matthew D and Sun, Chuangchuang and Abdulhai, Marwa and Habibi, Golnaz and Lopez-Cot, Sebastian and Tesauro, Gerald and How, Jonathan},
  booktitle = 	 {Proceedings of the 38th International Conference on Machine Learning},
  pages = 	 {5541--5550},
  year = 	 {2021},
  editor = 	 {Meila, Marina and Zhang, Tong},
  volume = 	 {139},
  series = 	 {Proceedings of Machine Learning Research},
  month = 	 {18--24 Jul},
  publisher =    {PMLR}
}

@ARTICLE{MLiu_PotentialGame,
  author={Liu, Mushuang and Kolmanovsky, Ilya and Tseng, H. Eric and Huang, Suzhou and Filev, Dimitar and Girard, Anouck},
  journal={IEEE Transactions on Intelligent Transportation Systems}, 
  title={Potential Game-Based Decision-Making for Autonomous Driving}, 
  year={2023},
  volume={24},
  number={8},
  pages={8014-8027},
  doi={10.1109/TITS.2023.3264665}}

@book{marl-book,
  author = {Stefano V. Albrecht and Filippos Christianos and Lukas Sch\"afer},
  title = {Multi-Agent Reinforcement Learning: Foundations and Modern Approaches},
  publisher = {MIT Press},
  year = {2024},
  url = {https://www.marl-book.com}
}

@book{Rajamani2012,
  author    = {Rajamani, Rajesh},
  title     = {Vehicle Dynamics and Control},
  edition   = {2nd},
  publisher = {Springer},
  address   = {New York, NY, USA},
  year      = {2012},
  isbn      = {978-1-4614-1432-2},
  doi       = {10.1007/978-1-4614-1433-9}
}

@article{PotentialGames,
title = {Potential Games},
journal = {Games and Economic Behavior},
volume = {14},
number = {1},
pages = {124-143},
year = {1996},
issn = {0899-8256},
doi = {https://doi.org/10.1006/game.1996.0044},
author = {Dov Monderer and Lloyd S. Shapley},
}

@inproceedings{MA-Meta-RL,
 author = {Mao, Weichao and Qiu, Haoran and Wang, Chen and Franke, Hubertus and Kalbarczyk, Zbigniew and Iyer, Ravishankar and Basar, Tamer},
 booktitle = {Advances in Neural Information Processing Systems},
 doi = {10.52202/075280-2906},
 pages = {66556--66570},
 publisher = {Curran Associates, Inc.},
 title = {Multi-Agent Meta-Reinforcement Learning: Sharper Convergence Rates with Task Similarity},
 volume = {36},
 year = {2023}
}

@InProceedings{Think2Drive,
author="Li, Qifeng
and Jia, Xiaosong
and Wang, Shaobo
and Yan, Junchi",
title="Think2Drive: Efficient Reinforcement Learning by Thinking with Latent World Model for Autonomous Driving (in {CARLA}-V2)",
booktitle="Computer Vision -- ECCV 2024",
year="2025",
publisher="Springer Nature Switzerland",
address="Cham",
pages="142--158",
isbn="978-3-031-72995-9"
}

@inproceedings{EfficientRLforAD,
  title={Efficient Reinforcement Learning for Autonomous Driving with Parameterized Skills and Priors},
  author={Wang, Letian and Liu, Jie and Shao, Hao and Wang, Wenshuo and Chen, Ruobing and Liu, Yu and Waslander, Steven L.},
  booktitle={Robotics: Science and Systems XIX},
  year={2023},
  doi={10.15607/RSS.2023.XIX.102}
}

@book{putermanMarkov,
  author    = {Martin L. Puterman},
  title     = {Markov Decision Processes: Discrete Stochastic Dynamic Programming},
  publisher = {John Wiley \& Sons},
  address   = {New York, NY, USA},
  year      = {1994}
}

@InProceedings{PEARL,
  title = 	 {Efficient Off-Policy Meta-Reinforcement Learning via Probabilistic Context Variables},
  author =       {Rakelly, Kate and Zhou, Aurick and Finn, Chelsea and Levine, Sergey and Quillen, Deirdre},
  booktitle = 	 {Proceedings of the 36th International Conference on Machine Learning},
  pages = 	 {5331--5340},
  year = 	 {2019},
  volume = 	 {97},
  series = 	 {Proceedings of Machine Learning Research},
  month = 	 {09--15 Jun},
  publisher =    {PMLR},
}

@ARTICLE{InterpretableEnd2End,
  author={Chen, Jianyu and Li, Shengbo Eben and Tomizuka, Masayoshi},
  journal={IEEE Transactions on Intelligent Transportation Systems}, 
  title={Interpretable End-to-End Urban Autonomous Driving With Latent Deep Reinforcement Learning}, 
  year={2022},
  volume={23},
  number={6},
  pages={5068-5078},
  doi={10.1109/TITS.2020.3046646}}

@article{DroneRacing,
  author  = {Kaufmann, Elia and Bauersfeld, Leonard and Loquercio, Antonio and M{\"u}ller, Matthias and Koltun, Vladlen and Scaramuzza, Davide},
  title   = {Champion-level drone racing using deep reinforcement learning},
  journal = {Nature},
  year    = {2023},
  volume  = {620},
  number  = {7976},
  pages   = {982--987},
  doi     = {10.1038/s41586-023-06419-4}
}

@ARTICLE{MPGTVT,
  author={Yan, Huiwen and Liu, Mushuang},
  journal={IEEE Transactions on Vehicular Technology}, 
  title={Markov Potential Game and Multi-Agent Reinforcement Learning for Autonomous Driving}, 
  year={2026},
  volume={},
  number={},
  pages={1-14}}

@ARTICLE{gradientPlay,
  author={Zhang, Runyu and Ren, Zhaolin and Li, Na},
  journal={IEEE Transactions on Automatic Control}, 
  title={Gradient Play in Stochastic Games: Stationary Points, Convergence, and Sample Complexity}, 
  year={2024},
  volume={69},
  number={10},
  pages={6499-6514},
  doi={10.1109/TAC.2024.3387208}}

@book{fudenberg1991game,
  title={Game theory},
  author={Fudenberg, Drew},
  year={1991},
  publisher={MIT press}
}

@article{RL_communication,
  title={Learning and uncertainty-exploited directional antenna control for robust long-distance and broad-band aerial communication},
  author={Liu, Mushuang and Wan, Yan and Li, Songwei and Lewis, Frank L and Fu, Shengli},
  journal={IEEE Transactions on Vehicular Technology},
  volume={69},
  number={1},
  pages={593--606},
  year={2019},
  publisher={}
}

@misc{MAML,
      title={Model-Agnostic Meta-Learning for Fast Adaptation of Deep Networks}, 
      author={Chelsea Finn and Pieter Abbeel and Sergey Levine},
      year={2017},
      eprint={1703.03400},
      archivePrefix={arXiv},
      primaryClass={cs.LG},
      url={https://arxiv.org/abs/1703.03400}, 
}
\end{document}